\documentclass[conference]{IEEEtran}
\usepackage{amsmath,amssymb,amsfonts,amsthm}
\usepackage{booktabs}
\usepackage{array}
\newcolumntype{L}[1]{>{\raggedright\arraybackslash}p{#1}}
\usepackage{multirow}
\usepackage{graphicx}
\usepackage[superscript]{cite}
\makeatletter\def\@citess#1{\textsuperscript{[#1]}}\makeatother  
\usepackage{algorithm}
\usepackage{algpseudocode}
\usepackage{tikz}
\usetikzlibrary{arrows.meta,positioning}
\usepackage{arfig}
\usepackage{microtype}
\usepackage{enumitem}
\usepackage{balance}
\usepackage{url}
\usepackage{soul}\sethlcolor{yellow}
\newif\ifmarkrev\markrevfalse  
\newcommand{\rev}[1]{#1}  
\newcommand{\revb}[1]{\ifmarkrev\hl{#1}\else#1\fi}  
\input{glyphtounicode}\pdfgentounicode=1

\newcommand{\method}{\textsc{CASEL}}
\renewcommand{\figurename}{Figure}\renewcommand{\tablename}{Table}\renewcommand{\thetable}{\arabic{table}}
\newcommand{\E}{\mathbb{E}}
\newcommand{\Prb}{\mathbb{P}}
\newtheorem{proposition}{Proposition}
\setlist[itemize]{leftmargin=*,nosep}
\setlist[enumerate]{leftmargin=*,nosep}

\title{CASEL: Consequence-Aware Selective Execution for Reliable Long-Horizon Tool-Using Agents}
\makeatletter\newcommand{\linebreakand}{\end{@IEEEauthorhalign}\hfill\mbox{}\par\vspace{6pt}\mbox{}\hfill\begin{@IEEEauthorhalign}}\makeatother  
\author{\IEEEauthorblockN{1\textsuperscript{st} Yanyan Zhang}
\IEEEauthorblockA{\textit{Carnegie Mellon University}\\
Pittsburgh, PA 15213, USA\\
zhangyanyanusa@gmail.com}
\and
\IEEEauthorblockN{2\textsuperscript{nd} Fei Shen}
\IEEEauthorblockA{\textit{Feinity LLC}\\
Sacramento, CA 95816, USA\\
shenfei2031@hotmail.com}
\and
\IEEEauthorblockN{3\textsuperscript{rd} Ting Yan}
\IEEEauthorblockA{\textit{Computer Science Department}\\
\textit{Dartmouth College}\\
Hanover, NH 03755, USA\\
ting.yan.gr@dartmouth.edu}
\linebreakand
\IEEEauthorblockN{4\textsuperscript{th} Hao Zhang\textsuperscript{*}}
\IEEEauthorblockA{\textit{Independent Researcher}\\
Los Angeles, CA 90095, USA\\
\textsuperscript{*}Corresponding author: alex1995101099@gmail.com}
}

\begin{document}
\maketitle

\begin{abstract}
Tool-using language agents increasingly modify external systems, yet most runtime safeguards allocate scrutiny by generic uncertainty scores or coarse mutation heuristics. We introduce \method{} (Consequence-Aware Selective Execution Layer), a model-agnostic runtime allocation policy that combines calibrated failure probability, state-dependent consequence, reversibility, ambiguity, and evidence freshness to choose, for each proposed tool call, among direct execution, verification, sandboxing, clarification, approval, and abstention by minimizing a per-mode cost. We formalize direct execution as selective prediction and state and test the validity conditions of a finite-sample threshold and a trajectory risk ledger. In a controlled simulation over 20 seeds with 120,000 test actions each, in which every method has the same intervention menu and the same compute cost, \method{} reduces harmful-side-effect trajectories from 14.75 percent for an uncertainty-ranked gate to 7.02 percent and mean normalized harm from 0.184 to 0.123. A gate ranked by the product of failure probability and consequence reaches similar harm (5.82 percent, 0.120) but sends 45 percent more actions to human approval, and removing reversibility-aware routing raises the rate to 15.53 percent. A trajectory-level risk budget, in contrast, increases harm at equal cost (11.82 percent). Intervention effectiveness is assumed, not measured: the results rank allocation rules under the simulator's assumptions and do not establish deployment effectiveness.
\end{abstract}


\section{Introduction}
Language-model agents increasingly act on external systems over long trajectories, \rev{and as models grow more capable they become autonomous domain-specific agents, e.g., SP-Mind runs complete spatial-proteomics analyses with specialized tools.\mbox{\cite{yuan2026spmind}}} \rev{They now also generate device configurations for multi-vendor optical transport networks,\mbox{\cite{xiao2026yang}} where an erroneous change can affect live services.} Benchmarks such as WebArena, $\tau$-Bench/$\tau^2$-Bench, and AgentDojo expose failures from stale state, planning errors, unsafe tool use, and adversarial observations \cite{zhou2024webarena,yao2025taubench,barres2025tau2,debenedetti2024agentdojo}. Reflection and self-critique improve reasoning \cite{yao2023react,shinn2023reflexion}, and recent work studies clarification, abstention, mutating-action verification, least privilege, and transactional containment \cite{suri2026sage,liu2026agentabstain,cuadron2025saber,betser2026agentrim,chen2026cordon}. These mechanisms address different failure sources but leave a common allocation question: \emph{how much scrutiny should be spent before a particular action is allowed to change the world?}

Uncertainty alone does not answer it. Two actions may share a $20\%$ probability of error, yet a read-only search is cheap to retry while an irreversible transfer is costly; an uncertainty-only gate treats them alike, and a mutation-only gate still conflates a reversible edit with an irreversible commitment.

\def\exP{0.20}
\def\exLam{0.02}
\def\exMu{0.015}
\def\exParams{\begin{tabular}{@{}lrrrr@{}} & dir. & ver. & sbx. & appr. \\$\gamma_m$ & 1.00 & 0.42 & 0.15 & 0.04 \\$k(m)$ & 0 & 1 & 1.6 & 2.2 \\$q(m)$ & 0 & 0 & 0.04 & 0.65 \\\end{tabular}}
\def\exCA{0.05}
\def\exRA{0.010}
\def\exJA{\begin{tabular}{@{}lcc@{}}{\color{armuted}$J(m)$:} & direct & verify \\ & \textbf{0.010} & 0.024\end{tabular}}
\def\exMA{direct}
\def\exCB{0.40}
\def\exRB{0.080}
\def\exJB{\begin{tabular}{@{}lccc@{}}{\color{armuted}$J(m)$:} & direct & verify & sandbox \\ & 0.080 & 0.054 & \textbf{0.045}\end{tabular}}
\def\exMB{sandbox}
\def\exCC{0.90}
\def\exRC{0.180}
\def\exJC{\begin{tabular}{@{}lccc@{}}{\color{armuted}$J(m)$:} & direct & verify & approve \\ & 0.180 & 0.096 & \textbf{0.061}\end{tabular}}
\def\exMC{approve}

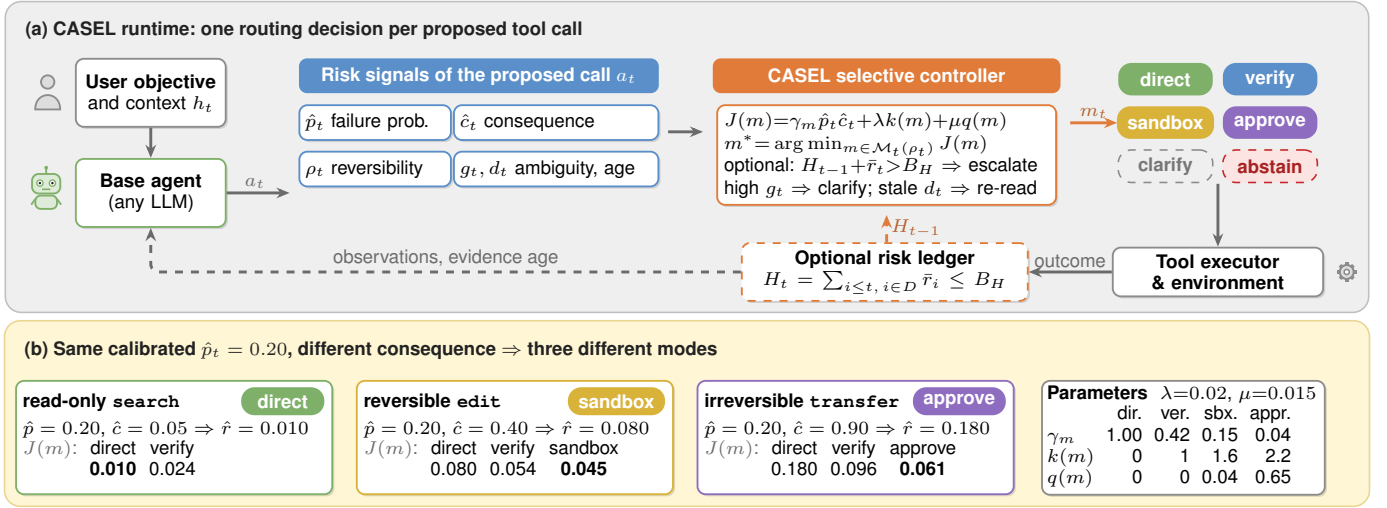
\begin{figure*}[t]
\centering
\begin{tikzpicture}[font=\sffamily\scriptsize,
  bx/.style={ar box=#1, font=\sffamily\scriptsize, inner sep=2.6pt},
  hd/.style={ar head=#1, font=\sffamily\bfseries\scriptsize, minimum height=4.6mm, inner ysep=1.5pt},
  chip/.style={ar box=blue, font=\sffamily\scriptsize, inner sep=2.2pt, text width=22mm, align=left, minimum height=4.6mm},
  pill/.style={ar tag=#1, font=\sffamily\bfseries\scriptsize, minimum width=12.5mm, minimum height=4.3mm},
  dpill/.style={pill=#1, fill=ar#1, draw=ar#1hd, dashed, line width=0.6pt, text=ar#1hd!85!black},
  lab/.style={font=\sffamily\scriptsize, text=armuted, inner sep=1.2pt},
]
  \node[bx=gray, text width=18mm, minimum height=9mm] (ctx) at (0,0) {\textbf{User objective}\\and context $h_t$};
  \node[bx=green, text width=18mm, minimum height=9mm, below=4mm of ctx] (agent) {\textbf{Base agent}\\(any LLM)};
  \pic[/arfig/icon=argrayhd, scale=0.55] at ($(ctx.west)+(-0.36,0)$) {ar person};
  \pic[/arfig/icon=argreenhd, scale=0.55] at ($(agent.west)+(-0.38,0.02)$) {ar robot};
  \draw[ar flow] (ctx) -- (agent);

  \node[hd=blue, anchor=north west, minimum width=47.6mm] (sh) at ($(ctx.north east)+(9.5mm,0)$) {Risk signals of the proposed call $a_t$};
  \node[chip, text width=18.5mm, anchor=north west] (s1) at ($(sh.south west)+(0,-1.4mm)$) {$\hat p_t$\ failure prob.};
  \node[chip, text width=25.5mm, anchor=north east] (s2) at ($(sh.south east)+(0,-1.4mm)$) {$\hat c_t$\ consequence};
  \node[chip, text width=18.5mm, below=1.3mm of s1] (s3) {$\rho_t$\ reversibility};
  \node[chip, text width=25.5mm, below=1.3mm of s2] (s4) {$g_t{,}\,d_t$\ ambiguity, age};
  \begin{scope}[on background layer]
    \node[ar panel=blue, fit=(sh)(s1)(s2)(s3)(s4), inner sep=3pt] (sig) {};
  \end{scope}
  \draw[ar flow] (agent.east) -- node[above, lab, pos=0.45] {$a_t$} (agent.east -| sig.west);

  \node[hd=peach, fill=arours, anchor=north west, minimum width=46mm] (ch) at ($(sh.north east)+(7mm,0)$) {CASEL selective controller};
  \node[bx=peach, draw=arours, anchor=north, text width=43mm, align=left, below=1.4mm of ch] (c1)
       {$J(m){=}\gamma_m\hat p_t\hat c_t{+}\lambda k(m){+}\mu q(m)$\\[0.5pt]
        $m^\ast{=}\arg\min_{m\in\mathcal M_t(\rho_t)}J(m)$\\[0.5pt]
        optional: $H_{t-1}{+}\bar r_t{>}B_H$ $\Rightarrow$ escalate\\[0.5pt]
        high $g_t$ $\Rightarrow$ clarify;\ stale $d_t$ $\Rightarrow$ re-read};
  \begin{scope}[on background layer]
    \node[ar panel=peach, draw=arours, line width=0.9pt, fit=(ch)(c1), inner sep=3pt] (ctl) {};
  \end{scope}
  \draw[ar flow] (sig.east |- ctl.center) -- (ctl.west);

  \node[pill=green, anchor=north west] (m1) at ($(ch.north east)+(7.5mm,-0.6mm)$) {direct};
  \node[pill=blue, right=1.2mm of m1] (m2) {verify};
  \node[pill=yellow, below=1.3mm of m1] (m3) {sandbox};
  \node[pill=lavender, below=1.3mm of m2] (m4) {approve};
  \node[dpill=gray, below=1.3mm of m3] (m5) {clarify};
  \node[dpill=red, below=1.3mm of m4] (m6) {abstain};
  \node[fit=(m1)(m6), inner sep=0pt] (modes) {};
  \draw[ar flow=arours] (ctl.east |- modes.center) -- node[above, lab, text=arours!80!black] {$m_t$} (modes.west);

  \node[bx=gray, text width=26mm, minimum height=6.4mm, anchor=north] (exec) at ($(modes.south |- ctl.south)+(0,-4.2mm)$)
       {\textbf{Tool executor \& environment}};
  \pic[/arfig/icon=argrayhd, scale=0.42] at ($(exec.east)+(0.3,0)$) {ar gear};
  \node[bx=peach, draw=arours, dashed, anchor=east, text width=36mm, minimum height=6.4mm] (led) at ($(exec.west)+(-11mm,0)$)
       {\textbf{Optional risk ledger}\ \ $H_t=\sum_{i\le t,\,i\in D}\bar r_i\le B_H$};
  \draw[ar flow] (modes.south) -- (exec.north);
  \draw[ar flow] (exec.west) -- node[above, lab] {outcome} (led.east);
  \draw[ar dflow=arours] (led.north -| ctl.center) -- node[right, lab, text=arours!80!black] {$H_{t-1}$} (ctl.south);
  \draw[ar dflow] (led.west) -| node[pos=0.25, above, lab] {observations, evidence age} (agent.south);

  \node[ar label, font=\sffamily\bfseries\scriptsize, anchor=south west] (ta) at ($(ctx.north west)+(-0.78,1.2mm)$)
       {(a) CASEL runtime: one routing decision per proposed tool call};
  \begin{scope}[on background layer]
    \node[ar panel=gray, fit={(ta)(ctx)(agent)(sig)(ctl)(modes)(exec)(led)($(ctx.west)+(-0.78,0)$)($(exec.east)+(0.45,0)$)}, inner sep=4pt] (pa) {};
  \end{scope}

  \node[ar label, font=\sffamily\bfseries\scriptsize, anchor=north west] (tb) at ($(pa.south west)+(4pt,-2.4mm)$)
       {(b) Same calibrated $\hat p_t=\exP$, different consequence $\Rightarrow$ three different modes};
  \newcommand{\excard}[6]{
    \node[bx=#4, anchor=north west, text width=40mm, align=left, minimum height=15mm] (#1) at #2
      {\textbf{#3}\\[2.2pt]
       $\hat p=\exP$,\ $\hat c=#5$\ $\Rightarrow$\ $\hat r=#6$\\[0.5pt]
       {\setlength{\tabcolsep}{2.4pt}\csname exJ#1\endcsname}};
    \node[pill=#4, anchor=north east, minimum width=11mm, minimum height=3.7mm, inner ysep=0.8pt] at ($(#1.north east)+(-1.1mm,-0.9mm)$) {\csname exM#1\endcsname};}
  \excard{A}{($(tb.south west)+(0,-1.6mm)$)}{read-only \texttt{search}}{green}{\exCA}{\exRA}
  \excard{B}{($(A.north east)+(3mm,0)$)}{reversible \texttt{edit}}{yellow}{\exCB}{\exRB}
  \excard{C}{($(B.north east)+(3mm,0)$)}{irreversible \texttt{transfer}}{lavender}{\exCC}{\exRC}
  \node[bx=gray, anchor=north west, text width=36mm, align=left, minimum height=15mm, inner sep=2pt] (uo) at ($(C.north east)+(3.5mm,0)$)
       {\textbf{Parameters}\ \ $\lambda{=}\exLam$,\ $\mu{=}\exMu$\\[0.5pt]
        {\setlength{\tabcolsep}{2.2pt}\exParams}};
  \begin{scope}[on background layer]
    \node[ar panel=yellow, fit={(tb)(A)(B)(C)(uo)($(pa.west |- A.south)+(4pt,0)$)($(pa.east |- A.south)+(-4pt,0)$)}, inner sep=4pt] (pb) {};
  \end{scope}
\end{tikzpicture}
\caption{\textbf{\method{} routes every proposed tool call by consequence-weighted risk, not by uncertainty alone.}
(a) Structured risk signals feed a controller that minimizes the per-mode cost $J(m)$ over the modes feasible for the
action's reversibility class; \revb{dashed elements are not part of the main comparison; clarify is tested in \protect\mbox{Sec.~\ref{sec:checks}}, abstain is not exercised.} (b) Worked example with the
nominal parameters (inset): three calls with the same $\hat p_t$ but different $\hat c_t$ go to direct execution,
sandboxing, and approval (lowest $J$ in bold); an uncertainty-only gate cannot tell them apart.}
\label{fig:arch}
\end{figure*}

We propose \method{}, a runtime layer between an arbitrary base agent and its tool executor (Figure~\ref{fig:arch}). At each step the controller estimates failure probability $\hat p_t$, consequence $\hat c_t$, reversibility $\rho_t$, ambiguity $g_t$, and evidence staleness $d_t$, and selects a graded intervention. Scrutiny should scale with consequence-weighted risk rather than generic confidence, the intervention should match the action's reversibility, and direct execution should be calibrated as selective prediction \cite{chow1970reject,geifman2019selectivenet,angelopoulos2024crc}. \method{} is therefore not another uncertainty gate but a runtime \emph{allocation policy} that decides where scrutiny is worth its cost and calls existing safety mechanisms as its intervention branches.

Our contributions are: (i) a consequence-aware execution objective with six intervention modes, \revb{five of which are exercised (abstain and the staleness signal $d_t$ are not)}; (ii) explicit validity conditions for a finite-sample direct-execution threshold, \revb{shown to bind and hold over a sweep of $\epsilon$, and an analysis of a trajectory risk ledger whose premise fails in practice}; and (iii) a reproducible controlled evaluation in which all comparators share the intervention menu and cost, with ablations, paired significance tests, stress tests, and a negative result for trajectory budgeting, plus a deployment path.

\section{Related Work}
\textbf{Agents and benchmarks.} ReAct interleaves reasoning with action and Reflexion stores verbal feedback across attempts \cite{yao2023react,shinn2023reflexion}; WebArena and the $\tau$-family emphasize stateful, long-horizon tool use \cite{zhou2024webarena,yao2025taubench,barres2025tau2}. Focus-Then-Contact learns robot skills with residual RL under timed human intervention \cite{qiao2026focus}, and DRAFT-RL scores multi-agent reasoning drafts during RL \cite{li2026draftrl}. \rev{Domain-specific agents bring domain-specific benchmarks, such as SP-Bench with 102 spatial-proteomics tasks.\mbox{\cite{yuan2026spmind}}} These works strengthen base policies but do not decide when an action needs extra evidence.

\textbf{Clarification, abstention, and tool safety.} SAGE-Agent selects clarification questions by value of information \cite{suri2026sage}; AgentAbstain tests whether agents recognize when acting is unjustified \cite{liu2026agentabstain}. AgentDojo targets untrusted tool outputs,\cite{debenedetti2024agentdojo} which multi-agent evidence checking can screen;\cite{li2026ragfakenews} AgenTRIM enforces adaptive least privilege \cite{betser2026agentrim}, SABER safeguards mutating steps \cite{cuadron2025saber}, and Cordon stages multi-step effects as semantic transactions \cite{chen2026cordon}.

\textbf{Selective prediction and calibration.} Reject-option prediction formalizes an accuracy--coverage trade-off \cite{chow1970reject,geifman2019selectivenet}; temperature scaling calibrates scores \cite{guo2017calibration}, and conformal risk control bounds monotone risks in expectation \cite{angelopoulos2024crc}, weaker than the high-probability bound used here. \rev{RAGScope routes retrieval-augmented answers to acceptance, review, or an expensive verifier with a cheap, cost-aware evidence gate,\mbox{\cite{liu2026ragscope}}} a pass/escalate decision that \method{} generalizes to several interventions.

Table~\ref{tab:compare} summarizes the comparison, including orthogonal layers: REFLEX calls a strong LLM when its fast decision layer is unsure \cite{wu2026reflex}, and CA2I scores whole action sequences before execution \cite{jafari2026ca2i}. \method{}'s product risk signal makes uncertainty-, consequence-, and mutation-only gates (evaluated in Sec.~\ref{sec:eval}) special cases, and its heterogeneous interventions make the specialized mechanisms candidate branches rather than competitors.

\begin{table*}[t]
\centering
\caption{Qualitative comparison with alternative approaches. \emph{Step}: one decision per proposed tool call.}
\label{tab:compare}
\scriptsize
\setlength{\tabcolsep}{3pt}
\renewcommand{\arraystretch}{0.95}
\begin{tabular}{@{}L{2.9cm}L{3.0cm}L{2.0cm}L{2.55cm}L{2.2cm}L{4.4cm}@{}}
\toprule
Approach & Problem focus & Granularity & Risk signal & Intervention space & Relation to \method{}\\
\midrule
SAGE-Agent \cite{suri2026sage} & Underspecified intent & Question & Spec.\ vs.\ model unc. & Clarify & Implements \textsf{clarify}; triggered by $g_t$\\
REFLEX \cite{wu2026reflex} & Compute-efficient decisions & Decision & Model confidence & Call strong LLM & Orthogonal: chooses the decision maker\\
CA2I \cite{jafari2026ca2i} & Pre-execution outcomes & Action sequence & Predicted return, risk & Select a plan & Orthogonal: admits whole plans\\
\midrule
\textbf{\method{}} & Scrutiny before commit & Step & $\hat p_t\hat c_t,\rho_t,g_t,d_t$ & Six modes, cost-aware & ---\\
\bottomrule
\end{tabular}
\end{table*}

\section{Problem Formulation}
At step $t$, history $h_t$ holds the user objective, observations, and prior actions, and a base policy proposes a tool call $a_t=(z_t,\theta_t)$. Let $F_t\in\{0,1\}$ indicate that $a_t$ is invalid relative to the task, policy, tool contract, or environment state. Its failure probability is
\begin{equation}
 p_t=\Prb(F_t=1\mid h_t,a_t,s_t),
\end{equation}
where $s_t$ is the environment state when $a_t$ would be committed, observed only partially through $h_t$. With $c_t\in[0,1]$ the normalized consequence if the action is wrong, the one-step expected harmful loss is $r_t=p_tc_t$.
\textbf{Notation.} Unhatted $p_t,c_t,r_t$ are true, unobserved quantities used in the analysis; hatted $\hat p_t,\hat c_t,\hat r_t=\hat p_t\hat c_t$ are run-time estimates, and Algorithm~\ref{alg:casel} and all routing rules use only these. $\bar r_t$ is the score charged to the optional ledger. \revb{Throughout, including \mbox{Figure~\ref{fig:arch}} and \mbox{Table~\ref{tab:compare}}, $\lambda$ and $\mu$ weight compute and interaction cost, $\rho_t$ is the reversibility class, and $\gamma_m$ the residual multiplier of mode $m$.}

Before commitment, a controller chooses $m_t$ from $\mathcal M=\{\textsf{direct},\textsf{verify},\textsf{sandbox},\textsf{clarify},\textsf{approve},\textsf{abstain}\}$. Mode $m$ has residual failure probability $p_t^{(m)}$, compute cost $k(m)$, and interaction cost $q(m)$, and the ideal myopic decision is
\begin{equation}
 m_t^*=\arg\min_{m\in\mathcal M_t}\left[p_t^{(m)}c_t+\lambda k(m)+\mu q(m)\right], \label{eq:mode}
\end{equation}
where $\mathcal M_t$ removes infeasible modes. For committed actions, linearity of expectation gives, without independence,
\begin{equation}
 \E\!\left[\textstyle\sum_{t=1}^{T}c_tF_t\right]=\textstyle\sum_{t=1}^{T}p_tc_t,
 \label{eq:cumrisk}
\end{equation}
so cumulative risk can be tracked by a ledger; Sec.~\ref{sec:eval} tests whether budgeting it helps.

\section{CASEL Runtime}
\subsection{Structured Risk Signals}
\textbf{Failure probability.} A raw score $u_t$ is calibrated on a held-out split, $\hat p_t=f_{\mathrm{cal}}(u_t)$, by isotonic regression or, given logits, temperature scaling \cite{guo2017calibration}; thresholds and budgets are then read as probabilities. $u_t$ can also draw on model internal states, as MISO does for ranking models \cite{zhang2026miso}.

\textbf{Consequence.} The controller computes
\begin{equation}
 \hat c_t=\mathrm{clip}\big(c_{\rm base}(z_t)\,M_{\rm scope}(\theta_t,h_t),\,0,\,1\big),
 \label{eq:chain}
\end{equation}
where $c_{\rm base}$ is a frozen per-tool severity from auditable metadata (write access, privilege, data sensitivity, value transfer) and $M_{\rm scope}\ge0$ a bounded magnitude factor (records, amount, rollback cost).

\textbf{Reversibility.} The class $\rho_t$ (read, reversible, compensable, or irreversible) changes the feasible set $\mathcal M_t(\rho_t)$ rather than scaling $\hat c_t$. \textsf{Direct}, \textsf{verify}, \textsf{clarify}, and \textsf{abstain} are always feasible; \textsf{sandbox} requires a reversible mutation or a dry-run interface such as the forkable environments of Counterfactual Rollout Replay \cite{li2026crr}, never for irreversible commitments; \textsf{approve} applies to compensable and irreversible actions.

\textbf{Ambiguity and freshness.} $g_t$ captures missing intent that more sampling cannot recover and routes to clarification; $d_t$ captures stale evidence and triggers a re-read. For GUI agents, perception errors (object detection \cite{kang2024lpdetr}, graphic understanding \cite{deng2025covis}) propagate into $u_t$ and $d_t$.

\subsection{Adaptive Intervention}
For feasible mode $m$, let $\gamma_m(\xi_t)\in[0,1]$ be the empirically estimated residual-error ratio, so $\hat p_t^{(m)}=\gamma_m(\xi_t)\hat p_t$, where $\xi_t$ collects the step features that determine how effective $m$ is (tool $z_t$, reversibility $\rho_t$, available checker); the simulator uses one constant $\gamma_m$ per mode. Substituting into Eq.~\eqref{eq:mode} gives $J(m)=\gamma_m\hat r_t+\lambda k(m)+\mu q(m)$. The expected harm reduction of $m$ is $V_m=(1-\gamma_m)\hat r_t$, so, ignoring interaction cost, $m$ beats direct execution when $V_m>\lambda k(m)$.

\begin{algorithm}[t]
\caption{Consequence-aware selective execution}\label{alg:casel}
\begin{algorithmic}[1]\footnotesize
\Require offline: $f_{\rm cal}$, threshold $\hat\tau$ (Eq.~\eqref{eq:select}); at step $t$: $h_t$, $a_t$, ledger $H_{t-1}$
\State $\hat p_t\gets f_{\rm cal}(u_t)$; estimate $\hat c_t,\rho_t,g_t,d_t$; $\hat r_t\gets\hat p_t\hat c_t$
\State \textbf{if} $g_t>\tau_g$ \textbf{return} (\textsf{clarify}, $H_{t-1}$)
\State \textbf{if} $d_t>\tau_d$ and refresh exists \textbf{return} (\textsf{verify}, $H_{t-1}$)
\State $\mathcal M_t\gets\textsc{FeasibleModes}(\rho_t,a_t)$; \textbf{if} $\hat r_t>\hat\tau$: remove \textsf{direct}
\State $m^*\gets\arg\min_{m\in\mathcal M_t}\gamma_m\hat r_t+\lambda k(m)+\mu q(m)$
\State \emph{optional ledger:} \textbf{if} $m^*=\textsf{direct}$ and $H_{t-1}+\bar r_t>B_H$: $m^*\gets\textsc{Escalate}(\rho_t,\mathcal M_t)$
\State \Return $\big(m^*,\;H_{t-1}+\bar r_t\,\mathbf 1[m^*=\textsf{direct}]\big)$
\end{algorithmic}
\end{algorithm}

\section{Selective-Risk Control}
\subsection{Direct-Execution Set and Finite-Sample Threshold}
Let $X$ be an action record, $S(X)$ a scalar score, and $L(X)\in[0,1]$ a harmful-loss label such as $cF$. For threshold $\tau$, the direct-execution set is $A_\tau=\{X:S(X)\le\tau\}$, with selective risk $R(\tau)=\E[L\mid X\in A_\tau]$ and coverage $C(\tau)=\Prb(X\in A_\tau)$. The objective is to maximize coverage subject to $R(\tau)\le\epsilon$. With $n_\tau$ calibration examples in $A_\tau$ and empirical risk $\hat R(\tau)$, Hoeffding's inequality gives $R(\tau)\le\hat R(\tau)+\sqrt{\ln(1/\delta)/(2n_\tau)}$ with probability at least $1-\delta$ for a fixed threshold. For a fixed grid of $K$ thresholds, a union bound yields the selection rule
\begin{equation}
 \hat R(\tau)+\sqrt{\ln(K/\delta)/(2n_\tau)}\le\epsilon. \label{eq:select}
\end{equation}

\begin{proposition}[Selective-risk validity]
Fix the scorer $S$ and a $K$-element grid before the calibration records are drawn, and let the records $(X_i,L_i)$, $L_i\in[0,1]$, be i.i.d.\ from the deployment distribution of proposed actions and disjoint from the data used to fit $f_{\rm cal}$. Then, with probability at least $1-\delta$, every grid threshold with $n_\tau>0$ that satisfies Eq.~\eqref{eq:select} has $R(\tau)\le\epsilon$, the risk of executing all of $A_\tau$ directly.
\end{proposition}
\begin{proof}
Given the membership indicators, the $n_\tau$ losses in $A_\tau$ are i.i.d.\ with mean $R(\tau)$, so Hoeffding's bound holds for each $\tau$ at level $\delta/K$; a union bound covers the grid.
\end{proof}

\textbf{Scope.} (i) Actions in a trajectory are dependent (exchangeability is not enough), so calibration uses one record per trajectory. (ii) The bound covers $A_\tau$ as a whole, not a further-selected subset, which can concentrate high-loss actions. To cover the deployed policy, Eq.~\eqref{eq:select} is applied to the set that the complete policy $\pi_\tau$ executes directly, whose indicator is a fixed function of the trajectory prefix.

\subsection{Trajectory Ledger}
An optional ledger tracks $H_t=\sum_{i\le t}D_i\bar r_i$, where $D_i\in\{0,1\}$ indicates direct execution, and blocks direct execution that would exceed a budget $B_H$.

\begin{proposition}[Budgeted expected harm]
Let $\mathcal G_t$ be the information available when step $t$ is routed, with $D_t$ and $\bar r_t$ $\mathcal G_t$-measurable. If $\E[c_tF_t\mid\mathcal G_t]\le\bar r_t$ whenever $D_t=1$ and $\sum_tD_t\bar r_t\le B_H$ almost surely, then $\E[\sum_tD_tc_tF_t]\le B_H$.
\end{proposition}
\begin{proof}
By the tower property, $\E[\sum_tD_tc_tF_t]=\E[\sum_tD_t\E[c_tF_t\mid\mathcal G_t]]\le\E[\sum_tD_t\bar r_t]\le B_H$.
\end{proof}
This is conditional accounting, not a reliability guarantee: $\bar r_t$ must upper-bound the true conditional harm at every step, which a calibrated point estimate does not, and the residual harm of intervened actions is outside the ledger.

\textbf{Runtime use.} Offline, the operator selects the largest grid $\hat\tau$ whose complete-policy direct set satisfies Eq.~\eqref{eq:select} ($S=\hat r$, $L=cF$); online, a call runs directly only if $\hat r_t\le\hat\tau$ and \textsf{direct} minimizes $J(m)$. The main comparison uses $\hat\tau=\infty$; Sec.~\ref{sec:checks} tests the threshold rule and the ledger.

\section{Controlled Evaluation}\label{sec:eval}
\subsection{Protocol}\label{sec:protocol}
Each seed has an independent 40,000-action calibration split and 6,000 test trajectories of 20 actions (120,000 test actions). Actions are read-only ($58\%$), reversible ($28\%$), or irreversible ($14\%$). True failure probabilities come from class-dependent Beta distributions with a trajectory-level difficulty perturbation, creating within-task dependence; consequences lie mainly in $[0.02,0.20]$, $[0.15,0.70]$, and $[0.55,1]$ for the three classes. The controller sees a noisy raw failure score, isotonic-calibrated on the calibration split, and a noisy consequence estimate.

The simulator instantiates direct, verify, sandbox, and approve with residual multipliers $\gamma_m$ of $1.00$, $0.42$, $0.15$, and $0.04$, compute costs $k(m)$ of $0$, $1.0$, $1.6$, and $2.2$, and interaction costs $q(m)$ of $0$, $0$, $0.04$, and $0.65$. These are assumed constants. Direct and verify are always feasible, sandbox only for reversible mutations, and approve only for irreversible ones with $\hat c_t\ge0.55$. \revb{Clarify is exercised only in the ambiguity stress family (\mbox{Sec.~\ref{sec:checks}}); abstain and the staleness branch ($d_t$, $\tau_d$) are not exercised.} The cost weights keep $\lambda:\mu=0.02:0.015$; a common factor sets the operating point.

\textbf{Comparators.} \emph{Verify-only gates} (the usual scalar gates) verify the actions ranked highest by $\hat p$ or $\hat c$. \emph{Menu gates} rank by $\hat p$, $\hat c$, or $\hat p\hat c$ and give the top-ranked actions their strongest feasible intervention (sandbox if reversible, approve if irreversible, otherwise verify), so they have \method{}'s menu. \method{} minimizes $J(m)$. Ablations: \method{}-NoRev (verify only), \method{}+Budget ($\lambda=0.02$ with the ledger, $B_H=0.08$, $\bar r_t=\hat p_t\hat c_t$), and \method{}-Uncal (raw score). All of these are matched to one compute cost, $0.486$ (for \method{}, cost-weight factor $0.58\pm0.02$), using scores but no outcome labels; Base, Always-Verify, and Mutation-Gated are fixed policies.

\textbf{Metrics.} An action routed to $m$ fails with probability $\gamma_m p_t$, using one uniform draw per action shared by all methods (paired outcomes). \emph{Task success} is the fraction of trajectories with no failed action; the \emph{harmful side-effect rate} (HSR) is the fraction with at least one failed action of true consequence $c_t\ge0.45$; \emph{mean normalized harm} (MNH) is the mean over trajectories of $\sum_{t=1}^{20}c_tF_t$. Direct coverage is the share of directly executed actions, cost the mean $k(m)$ per action, and approvals the share of actions sent to human approval. Results average 20 seeds.

\begin{figure}[t]
\centering
\includegraphics[width=\columnwidth]{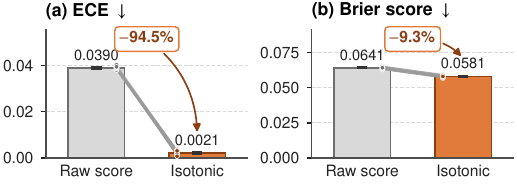}
\caption{\textbf{Isotonic calibration turns the raw failure score into a usable probability.} (a) ECE and (b) Brier score of the raw score and of the calibrated $\hat p$ on test actions. Bars: means over 20 seeds with 95\% CIs; lines: paired per-seed values.}
\label{fig:cal}
\end{figure}

\subsection{Main Results}
Isotonic calibration lowers mean ECE by $94.5\%$ and Brier score by $9.3\%$ (Figure~\ref{fig:cal}).

\begin{table*}[t]
\centering
\caption{Controlled evaluation over 20 seeds (task success and HSR: mean $\pm$ one standard deviation). All methods below Mutation-Gated have compute cost $0.486$. Menu: interventions a method may use (full = verify, sandbox, approve).}
\label{tab:results}
\footnotesize
\setlength{\tabcolsep}{5pt}
\renewcommand{\arraystretch}{0.95}
\begin{tabular}{llcccccc}
\toprule
Method (ranking) & Menu & Task Success $\uparrow$ & HSR $\downarrow$ & MNH $\downarrow$ & Direct Cov. & Cost & Approvals $\downarrow$\\
\midrule
Base & --- & $29.48\!\pm\!0.49$ & $32.76\!\pm\!0.70$ & 0.428 & 100.0\% & 0.000 & 0.0\%\\
Always-Verify & verify & $58.81\!\pm\!0.71$ & $15.47\!\pm\!0.53$ & 0.180 & 0.0\% & 1.000 & 0.0\%\\
Mutation-Gated & verify & $43.78\!\pm\!0.62$ & $15.47\!\pm\!0.53$ & 0.196 & 58.0\% & 0.420 & 0.0\%\\
Uncertainty gate ($\hat p$) & verify & $47.78\!\pm\!0.56$ & $19.25\!\pm\!0.57$ & 0.233 & 51.4\% & 0.486 & 0.0\%\\
Consequence gate ($\hat c$) & verify & $45.10\!\pm\!0.63$ & $15.47\!\pm\!0.53$ & 0.195 & 51.4\% & 0.486 & 0.0\%\\
\midrule
Uncertainty-Menu ($\hat p$) & full & $48.79\!\pm\!0.63$ & $14.75\!\pm\!0.41$ & 0.184 & 66.9\% & 0.486 & 6.9\%\\
Consequence-Menu ($\hat c$) & full & $43.54\!\pm\!0.50$ & $4.38\!\pm\!0.29$ & 0.136 & 74.7\% & 0.486 & 13.8\%\\
Risk-Menu ($\hat p\hat c$) & full & $46.26\!\pm\!0.50$ & $5.82\!\pm\!0.35$ & 0.120 & 74.2\% & 0.486 & 12.6\%\\
\textbf{\method{}} ($J(m)$) & full & $47.01\!\pm\!0.56$ & $7.02\!\pm\!0.30$ & 0.123 & 69.1\% & 0.486 & 8.7\%\\
\midrule
\method{}-NoRev & verify & $46.56\!\pm\!0.54$ & $15.53\!\pm\!0.52$ & 0.194 & 51.4\% & 0.486 & 0.0\%\\
\method{}+Budget & full & $45.82\!\pm\!0.64$ & $11.82\!\pm\!0.54$ & 0.167 & 63.9\% & 0.486 & 4.3\%\\
\method{}-Uncal & full & $47.83\!\pm\!0.50$ & $8.22\!\pm\!0.36$ & 0.129 & 69.3\% & 0.486 & 8.4\%\\
\bottomrule
\end{tabular}
\end{table*}

\textbf{Main result.} With the same menu and the same cost, \method{} lowers HSR from $14.75\%$ for Uncertainty-Menu to $7.02\%$ (paired difference $-7.73$ points, 95\% CI $[-7.87,-7.60]$) and MNH from $0.184$ to $0.123$; task success is $1.79$ points lower, because generic success weights every error equally (Table~\ref{tab:results}).

\textbf{Simple consequence-aware gates are strong.} Risk-Menu has slightly lower HSR ($5.82\%$) and MNH ($0.120$) than \method{}, and Consequence-Menu the lowest HSR ($4.38\%$) with higher MNH ($0.136$) and the lowest task success. Both send more actions to human approval: $12.6\%$ and $13.8\%$ versus $8.7\%$. Matched on total cost $\lambda k+\mu q$ instead, Risk-Menu has HSR $6.22\%$, MNH $0.126$, and task success $45.60\%$ (\method{}: $7.02\%$, $0.123$, $47.01\%$). The per-mode cost $J(m)$ therefore does not reduce harm beyond a product-score gate; it reaches similar harm with about $30\%$ fewer approvals and slightly higher task success, across the cost range (Figure~\ref{fig:pareto}(a,b)).

\begin{figure*}[t]
\centering
\includegraphics[width=\textwidth]{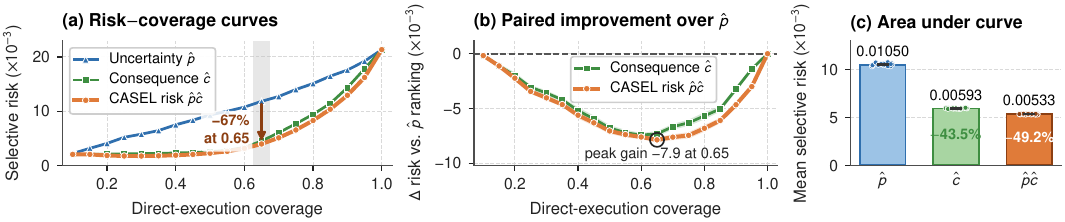}
\caption{\textbf{Ranking actions by $\hat p\hat c$ halves the selective-risk AUC of uncertainty-only ranking.} (a) Selective risk (mean consequence-weighted loss of directly executed actions) versus coverage when actions are admitted in increasing order of $\hat p$, $\hat c$, or $\hat p\hat c$; the arrow marks the largest gap. (b) Per-seed paired difference from the $\hat p$ ranking. (c) Area under each curve (average selective risk over coverage $0.1$--$1.0$); dots are seeds. Bands and whiskers are 95\% CIs over 20 seeds.}
\label{fig:rc}
\end{figure*}

\subsection{Risk--Coverage Quality}\label{sec:rc}
Figure~\ref{fig:rc} isolates score quality from the intervention model by ranking candidate direct actions by $\hat p$, $\hat c$, or $\hat p\hat c$. Averaged over coverage $0.1$--$1.0$, selective-risk AUC is $0.01050$, $0.00593$, and $0.00533$, respectively: the product reduces AUC by $49.2\%$ relative to uncertainty and by $10.1\%$ relative to consequence.

\begin{figure*}[t]
\centering
\includegraphics[width=\textwidth]{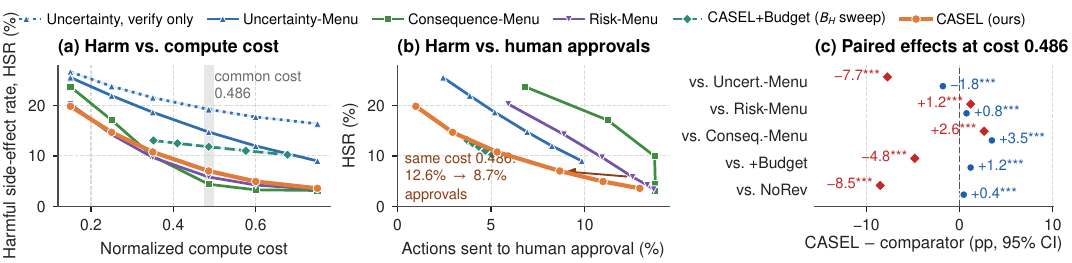}
\caption{\textbf{With equal menus, consequence-aware allocation has far lower harm than uncertainty ranking at every cost; \method{} matches a risk-ranked gate with fewer approvals, and the trajectory budget does not help.} (a) HSR versus compute cost; the band marks the common cost $0.486$. (b) HSR versus the share of actions sent to human approval. (c) Paired differences at cost $0.486$, \method{} minus comparator, in HSR (red diamonds) and task success (blue dots), in points; $^{***}$: $p<0.001$, two-sided paired $t$-test over the 20 seeds, Holm-corrected over the ten tests. Bands and bars are 95\% CIs.}
\label{fig:pareto}
\end{figure*}

\subsection{Ablations}
\textbf{Reversibility.} Restricting \method{} to verification raises HSR from $7.02\%$ to $15.53\%$ and MNH to $0.194$ at the same cost: matching the intervention to commitment semantics accounts for most of the gain. \textbf{Trajectory budget.} \method{}+Budget has HSR $11.82\%$ and MNH $0.167$ at the same compute cost and is dominated along its whole $B_H$ sweep (Figure~\ref{fig:pareto}(a)): a full ledger escalates low-risk actions late in a trajectory, where expected harm is small. In the harm--approval plane it lies on \method{}'s curve (Figure~\ref{fig:pareto}(b)): it trades compute for approvals rather than reducing harm, so the default configuration omits the ledger. \textbf{Calibration.} Using the raw score raises HSR by $1.20$ points and MNH by $0.005$ and raises task success by $0.82$ points. All ten differences in Figure~\ref{fig:pareto}(c) are significant (Holm-corrected $p<10^{-4}$).

\subsection{Checks of the Guarantees and Stress Tests}\label{sec:checks}
\textbf{Threshold rule.} Per seed we keep one random step from each of 40,000 fresh calibration trajectories (i.i.d.\ records) and apply Eq.~\eqref{eq:select} to the deployed policy (20-point grid, $\delta=0.1$). \revb{Sweeping $\epsilon$ over $0.006$--$0.016$: for $\epsilon\ge0.015$ the rule certifies the unrestricted policy (coverage $69.1\%$, true selective risk $0.0048$); at $\epsilon=0.014$ it binds in all 20 seeds ($\hat\tau=0.014$, coverage $62.3\%$, risk $0.0035$, cost $0.553$); below, the Hoeffding slack ($\approx0.010$) blocks certification. With $160{,}000$ records the rule binds at $\epsilon=0.009$ (coverage $63.1\%$) and certifies $8$ of 20 seeds at $0.008$ (coverage $52.6\%$). True selective risk never exceeded $\epsilon$.}

\textbf{Ledger.} \revb{Proposition~2 is therefore an analysis of when budgeting is sound, not a guarantee of the evaluated system:} for \method{}+Budget its premise fails on $54\%$ of directly executed steps ($p_tc_t>\bar r_t$). Realized direct-execution harm exceeds the budget \rev{by $10\%$ at $B_H=0.04$ and $2\%$ at $0.08$}, and residual harm from intervened actions ($0.086$, about half of MNH) is not covered.

\revb{\textbf{Ambiguity.} In $10\%$ of test actions intent is missing: with probability $a_t\sim U(0.2,0.6)$ a correctly executed call is still wrong. Verify and sandbox leave this error, approval halves it, and clarify ($k=0.2$, $q=0.3$) removes $95\%$ before $J(m)$ routes the call. With $g_t=a_t+\mathcal N(0,0.08^2)$ and $\tau_g=0.2$, \method{} clarifies $9.8\%$ of actions (precision $94\%$, recall $92\%$) and lowers HSR from $19.1\%$ to $8.4\%$; at equal compute cost, \method{} without $g_t$ has $18.5\%$ and Risk-Menu $16.4\%$ ($7.8$--$12.4\%$ for $\tau_g\in[0.1,0.4]$).}

\textbf{Misspecification.} Halving $\hat p$, $\hat c$, or every intervention's effectiveness raises \method{}'s HSR to $16.6\%$, $18.8\%$, and $21.0\%$, still below Uncertainty-Menu at the same cost ($23.4\%$, $25.9\%$, $24.3\%$) and within $0.6$ points of Risk-Menu: the ordering is stable, but absolute harm depends on $\gamma_m$ and unbiased estimates.

\section{Deployment Considerations}\label{sec:deploy}
\textbf{Adaptation and drift.} Only the signals in $c_{\rm base}$, $M_{\rm scope}$ and the preferred high-risk modes are domain-specific; \rev{in optical networks, for instance, alarm correlation and root-alarm detection\mbox{\cite{xiao2025joint}} show how far a fault propagates and can inform the scope factor}. $\lambda,\mu$ are read off the deployment's own cost--harm frontier (Figure~\ref{fig:pareto}), and $\gamma_m$ is estimated conservatively on trajectory-disjoint traces. \rev{Such gates must be recalibrated per domain: RAGScope's gate transfers poorly across sources but recovers with 100--200 target labels.\mbox{\cite{liu2026ragscope}}} When logged calibration or residual-risk drift leaves its tolerance region, the controller re-selects $\hat\tau$, re-estimates $\gamma_m$, or lowers $\lambda$ and $\mu$.

\textbf{Integration with existing AI tooling.} \method{} needs only an interception point for tool calls (an MCP proxy, a pre-execution hook, or a policy engine) and no retraining. \rev{In network operations, robotic test-automation platforms for optical transport and commercial O-RUs\mbox{\cite{xiao2025optical,xiao2026oru}} automate validation runs that can back the verify and sandbox branches.} Existing tools implement the branches: an LLM judge or an execution-grounded reward model such as BoostAPR's repair assessors \cite{li2026boostapr} for \textsf{verify}, Cordon for \textsf{sandbox}, an approval service for \textsf{approve}, and SAGE-Agent for \textsf{clarify} \cite{chen2026cordon,suri2026sage}.

\section{Discussion and Limitations}
\method{} is an allocation policy, not a safety guarantee; the simulation ranks allocation rules under assumed intervention effectiveness and costs, not their effect in deployment.

\textbf{Roadmap.} Planned (not yet run) evaluations estimate $\gamma_m$ from real verifier and rollback outcomes on $\tau^2$-Bench, WebArena, AgentDojo, AgentAbstain, and ToolEmu \cite{barres2025tau2,zhou2024webarena,debenedetti2024agentdojo,liu2026agentabstain,ruan2024toolemu}\rev{, and on domain benchmarks such as SP-Bench,\mbox{\cite{yuan2026spmind}} where consequence rubrics must be domain-specific}. \method{} was not compared experimentally with REFLEX \cite{wu2026reflex} or CA2I \cite{jafari2026ca2i}, which act at other levels.

\section{Conclusion}
In a controlled simulation with equal menus and costs, allocating scrutiny by consequence-weighted risk and reversibility roughly halves severe side effects relative to uncertainty ranking; a per-mode cost matches a risk-ranked gate with fewer approvals, and a trajectory budget did not help.

\end{document}